\documentclass{article} 
\usepackage{nips14submit_e,times}
\usepackage{hyperref}
\usepackage{url}
\usepackage{subfig}
\usepackage{amsmath}
\usepackage{amssymb}
\usepackage{amsthm}
\usepackage[algo2e,ruled]{algorithm2e}
\usepackage{dsfont}
\usepackage{graphicx}

\newcommand{\ind}{\mathds{1}}

\newcommand{\E}{\mathbb{E}}
\renewcommand{\P}{\mathbb{P}}

\newcommand{\De}{\Delta}

\newcommand{\var}{\mathrm{Var}}
\newcommand{\BB}{\mathcal{B}}

\newcommand{\tc}{T_\text{learn}}

\newcommand{\En}{\mathcal{E}_\text{good}}

\newtheorem{example}{Example}
\newtheorem{theorem}{Theorem}
\newtheorem{lemma}{Lemma}

\title{A Latent Source Model for \\ Online Collaborative Filtering}

\author{
Guy Bresler \qquad George H.~Chen \qquad Devavrat Shah \\
Department of Electrical Engineering and Computer Science \\
Massachusetts Institute of Technology\\
Cambridge, MA 02139 \\
\texttt{\{gbresler,georgehc,devavrat\}@mit.edu}
}

\nipsfinalcopy 

\begin{document}

\maketitle

\begin{abstract}
Despite the prevalence of collaborative filtering in recommendation systems,
there has been little theoretical development on why and how well it works,
especially in the ``online'' setting, where items are recommended to users
over time. We address this theoretical gap by introducing a model for online
recommendation systems, cast item recommendation under the model as a learning
problem, and analyze the performance of a cosine-similarity collaborative
filtering method. In our model, each of $n$ users either likes or dislikes
each of $m$ items. We assume there to be $k$ types of users, and all the users
of a given type share a common string of probabilities determining the chance
of liking each item. At each time step, we recommend an item to each user,
where a key distinction from related bandit literature is that once a user
consumes an item (e.g., watches a movie), then that item cannot be recommended
to the same user again. The goal is to maximize the number of likable items
recommended to users over time. Our main result establishes that after nearly
$\log(km)$ initial learning time steps, a simple collaborative filtering
algorithm achieves essentially optimal performance without knowing $k$. The
algorithm has an exploitation step that uses cosine similarity and
two types of exploration steps, one to explore the space of items (standard in
the literature) and the other to explore similarity between users (novel to
this work).
\end{abstract}

\section{Introduction}
Recommendation systems have become ubiquitous in our lives, helping us filter
the vast expanse of information we encounter into small selections tailored to
our personal tastes. Prominent examples include Amazon recommending items to
buy, Netflix recommending movies, and LinkedIn recommending jobs. 
In practice, recommendations are often made via \emph{collaborative
filtering}, which boils down to recommending an item to a user by considering
items that other similar or ``nearby'' users liked. Collaborative filtering
has been used extensively for decades now including in the \mbox{GroupLens} news
recommendation system \cite{grouplens_1994}, Amazon's item recommendation
system \cite{amazon_2003}, the Netflix Prize winning algorithm by
\mbox{BellKor's} Pragmatic Chaos
\cite{bellkor_2009,bigchaos_2009,pragmatic_2009}, and a recent song
recommendation system \cite{msd_challenge_winner_2013} that won the Million
Song Dataset Challenge \cite{msd_challenge_2011}.

Most such systems operate in the ``online" setting, where items are constantly
recommended to users over time. In many scenarios, it does not make sense to
recommend an item that is already consumed. For example, once Alice watches a
movie, there's little point to recommending the same movie to her again, at
least not immediately, and one could argue that recommending unwatched movies
and already watched movies could be handled as separate cases. Finally, what
matters is whether a {\em likable} item is recommended to a user rather than
an {\em unlikable} one. In short, a good online recommendation system should
recommend different likable items continually over time.

Despite the success of collaborative filtering, there has been little
theoretical development to justify its effectiveness in the online setting. We
address this theoretical gap with our two main contributions in this paper.
First, we frame online recommendation as a learning problem that fuses the
lines of work on sleeping bandits and clustered bandits. We impose the
constraint that once an item is consumed by a user, the system can't recommend
the item to the same user again. Our second main contribution is to analyze a
cosine-similarity collaborative filtering algorithm. The key insight is our
inclusion of two types of exploration in the algorithm: (1) the standard
random exploration for probing the space of items, and (2) a novel ``joint''
exploration for finding different user types. Under our learning problem
setup, after nearly $\log(km)$ initial time steps, the proposed algorithm
achieves near-optimal performance relative to an oracle
algorithm that recommends all likable items first. The nearly logarithmic
dependence is a result of using the two different exploration types.
We note that the algorithm does not know $k$.


\noindent{\bf Outline.} We present our model and learning problem for online
recommendation systems in Section \ref{sec:model}, provide a collaborative
filtering algorithm and its performance guarantee in Section \ref{sec:algo},
and give the proof idea for the performance guarantee in Section
\ref{sec:proof}. An overview of experimental results is given in Section
\ref{sec:experiments}. We discuss our work in the context of prior
work in Section \ref{sec:discussion}.


\section{A Model and Learning Problem for Online Recommendations}
\label{sec:model}


We consider a system with $n$ users and $m$ items. At each time step, each
user is recommended an item that she or he hasn't consumed yet, upon which,
for simplicity, we assume that the user immediately consumes the item and
rates it $+1$ (like) or $-1$ (dislike).%
\footnote{%
In practice, a user could ignore the recommendation. To keep our exposition
simple, however, we stick to this setting that resembles song recommendation
systems like Pandora that per user continually recommends a single item at a
time. For example, if a user rates a song as ``thumbs down'' then we assign a
rating of $-1$ (dislike), and any other action corresponds to $+1$ (like).
%
%
}
The reward earned by the recommendation system up to any time step is the
total number of liked items that have been recommended so far across all
users. Formally, index time by $t\in\{1,2,\dots\}$, and users by
$u\in[n]\triangleq\{1,\dots,n\}$. Let $\pi_{ut}\in[m]\triangleq\{1,\dots, m\}$
be the item recommended to user $u$ at time $t$. Let
$Y_{ui}^{(t)}\in\{-1,0,+1\}$ be the rating provided by user $u$ for item $i$
up to and including time $t$, where 0 indicates that no rating has been given
yet. A reasonable objective is to maximize the expected reward $r^{(T)}$ up to
time~$T$:
\begin{equation*}
r^{(T)}
\triangleq
  \sum_{t=1}^T\sum_{u=1}^n \E[Y_{u \pi_{ut}}^{(T)}]
= \sum_{i=1}^m\sum_{u=1}^n \E[Y_{u i}^{(T)}].
\end{equation*}
The ratings are noisy: the latent item preferences for user $u$ are
represented by a length-$m$ vector
$p_u \in [0,1]^m$, where user
$u$ likes item $i$ with probability $p_{ui}$, independently across items. For
a user $u$, we say that item $i$ is {\em likable} if $p_{ui} > 1/2$ and
{\em unlikable} if $p_{ui} < 1/2$. To maximize the expected reward $r^{(T)}$,
clearly likable items for the user should be recommended before unlikable
ones.

In this paper, we focus on recommending likable items. Thus, instead of
maximizing the expected reward $r^{(T)}$, we aim to maximize the expected
number of \emph{likable} items recommended up to time $T$:
\begin{equation}
\label{eq:reward+}
r_+^{(T)}
\triangleq
  \sum_{t=1}^T\sum_{u=1}^n \E[X_{ut}]\,,
\end{equation}
where $X_{ut}$ is the indicator random variable for whether the item
recommended to user $u$ at time $t$ is likable, i.e., $X_{ut} = +1$ if
$p_{u \pi_{ut}} > 1/2$ and $X_{ut}=0$ otherwise. Maximizing $r^{(T)}$ and
$r_+^{(T)}$ differ since the former asks that we prioritize items according
to their probability of being liked.

Recommending likable items for a user in
an arbitrary order is sufficient for many real recommendation systems such as
for movies and music. For example, we suspect that users wouldn't actually
prefer to listen to music starting from the songs that their user type would
like with highest probability
to the ones their user type would like with lowest probability; 
instead, each user would listen to
songs that she or he finds likable, ordered such that there is sufficient
diversity in the playlist to keep the user experience interesting. We target
the modest goal of merely recommending likable items, in any order.
Of course,
if all likable items have the same probability of being liked and similarly 
for all unlikable items, 
then maximizing $r^{(T)}$ and $r_+^{(T)}$ are equivalent. 


The fundamental challenge is that to learn about a user's preference for
an item, we need the user to rate (and thus consume) the item. But then we
cannot recommend that item to the user again! Thus, the only
way to learn about a user's preferences is through collaboration,
or inferring from other users' ratings. Broadly, such inference is possible if
the users preferences are somehow related.


In this paper, we assume a simple structure for shared user preferences. We
posit that there are $k<n$ different types of users, where users of the same
type have identical item preference vectors. The number of types $k$
represents the heterogeneity in the population. For ease of exposition, in
this paper we assume that a user belongs to each user type with probability
$1/k$. We refer to this model as a \emph{latent source model}, where each user
type corresponds to a latent source of users. We remark that there is evidence
suggesting real movie recommendation data to be well modeled by clustering of
both users and items \cite{bctf_2009}. Our model only assumes clustering over
users.

Our problem setup relates to some versions of the multi-armed bandit problem.
A fundamental difference between our setup and that of the standard stochastic
multi-armed bandit problem \cite{thompson_bandit_origin,regret_analysis_book}
is that the latter allows each item to be recommended an infinite number of
times. Thus, the solution concept for the stochastic multi-armed bandit
problem is to determine the best item (arm) and keep choosing it
\cite{finite_time_stochastic_multiarmed_bandits}. This observation applies
also to ``clustered bandits" \cite{clustered_bandits}, which like our work
seeks to capture collaboration between users. On the other hand, sleeping
bandits \cite{sleeping_bandits} allow for the available items at each time
step to vary, but the analysis is worst-case in terms of which items are
available over time. In our setup, the sequence of items that are available is
not adversarial. Our model combines the collaborative aspect of clustered
bandits with dynamic item availability from sleeping bandits, where we impose
a strict structure on how items become unavailable.

\section{A Collaborative Filtering Algorithm and Its Performance Guarantee}
\label{sec:algo}

This section presents our algorithm \textsc{Collaborative-Greedy} and its
accompanying theoretical performance guarantee. The algorithm is syntactically
similar to the $\varepsilon$-greedy algorithm for multi-armed bandits
\cite{reinforcement_learning_textbook}, which explores items with probability
$\varepsilon$ and otherwise greedily chooses the best item seen so far based
on a plurality vote. In our algorithm, the greedy choice, or exploitation,
uses the standard cosine-similarity measure. The exploration, on the other
hand, is split into two types, a standard item exploration in which a user is
recommended an item that she or he hasn't consumed yet uniformly at random,
and a joint exploration in which all users are asked to provide a rating for
the next item in a shared, randomly chosen sequence of items. Let's fill in
the details.
%
%

\begin{algorithm2e}[b]
\DontPrintSemicolon
\caption{\textsc{Collaborative-Greedy}%
         \label{alg:collaborative-greedy}}
\KwIn{Parameters $\theta \in [0,1]$, $\alpha \in (0,4/7]$.}
Select a random ordering $\sigma$ of the items $[m]$. Define
\[
\varepsilon_R(n) = \frac{1}{n^\alpha},
\qquad
\text{and}
\qquad
\varepsilon_J(t) = \frac{1}{t^\alpha}.
\]
\For{time step $t=1, 2, \dots, T$}{
With prob.~$\varepsilon_R(n)$: (\textbf{random exploration})
for each user, recommend a random item that the user has not rated.\;
With prob.~$\varepsilon_J(t)$: (\textbf{joint exploration})
for each user, recommend the first item in $\sigma$ that the user has not
rated.\;
With prob.~$1 - \varepsilon_J(t) - \varepsilon_R(n)$:
(\textbf{exploitation}) for each user $u$, recommend an item $j$ that the user
has not rated and that maximizes score $\widetilde{p}_{uj}^{(t)}$, which
depends on threshold $\theta$.\;
}
\end{algorithm2e}

\noindent\textbf{Algorithm.}
At each time step $t$, either all the users are asked to explore, or 
an item is recommended to each user by choosing the item with the highest 
score for that user. The pseudocode is described in Algorithm
\ref{alg:collaborative-greedy}. 
There are two types of exploration: \emph{random exploration}, which is for
exploring the space of items, and \emph{joint exploration}, which helps to
learn about similarity between users. For a pre-specified rate
$\alpha \in (0,4/7]$, we set the probability of random exploration to be
$\varepsilon_R(n)=1/n^\alpha$ (decaying with the number of users), and the
probability of joint exploration to be $\varepsilon_J(t)=1/t^\alpha$
(decaying with time).%
\footnote{For ease of presentation, we set the two explorations to have
the same decay rate $\alpha$, but our proof easily extends to encompass
different decay rates for the two exploration types. Furthermore, the
constant $4/7 \ge \alpha$ is not special. It could be different and only
affects another constant in our proof.}

Next, we define user $u$'s score $\widetilde{p}_{ui}^{(t)}$ for item $i$ at time $t$.
Recall that we observe $Y_{ui}^{(t)} = \{-1,0,+1\}$ as user $u$'s rating for item $i$
up to time $t$, where $0$ indicates that no rating has been given yet. We define
\[
\widetilde{p}_{ui}^{(t)}
\triangleq
  \begin{cases}
    \frac{\textstyle
          \sum_{v \in \widetilde{\mathcal{N}}_u^{(t)}}
            \ind\{Y_{vi}^{(t)} = +1\}}
         {\textstyle
          \sum_{v \in \widetilde{\mathcal{N}}_u^{(t)}}
            \ind\{Y_{vi}^{(t)} \ne 0\}}
    & \text{if }\textstyle{%
                \sum_{v \in \widetilde{\mathcal{N}}_u^{(t)}}
                  \ind\{Y_{vi}^{(t)} \ne 0\} > 0}, \\
    1/2 & \text{otherwise},
  \end{cases}
\]
where the neighborhood of user $u$ is given by 
\begin{align*}
&\widetilde{\mathcal{N}}_u^{(t)}
\triangleq
  \{v \in[n] :
    \langle \widetilde{Y}_u^{(t)}, \widetilde{Y}_v^{(t)} \rangle 
    \ge
    \theta
    |\text{supp}(\widetilde{Y}_u^{(t)})
     \cap \text{supp}(\widetilde{Y}_v^{(t)})|\},
\end{align*}
and $\widetilde{Y}_u^{(t)}$ consists of the revealed ratings of user $u$
restricted to items that have been jointly explored. In other words,
\[
\widetilde{Y}_{ui}^{(t)}
=\begin{cases}
   Y_{ui}^{(t)} & \text{if item }i\text{ is jointly explored by time }t, \\
   0 & \text{otherwise}.
 \end{cases}
\]
%
%
The neighborhoods are defined precisely by cosine similarity with respect to jointed explored items. 
To see this, for users $u$ and $v$ with revealed ratings $\widetilde{Y}_u^{(t)}$ and $\widetilde{Y}_v^{(t)}$, let $\Omega_{uv} \triangleq \text{supp}(\widetilde{Y}_u^{(t)}) \cap \text{supp}(\widetilde{Y}_v^{(t)})$ be the support overlap of $\widetilde{Y}_u^{(t)}$ and $\widetilde{Y}_v^{(t)}$, 
and let $\langle \cdot, \cdot \rangle_{\Omega_{uv}}$ be the dot product restricted to
entries in $\Omega_{uv}$. Then
\[
\frac{ \langle \widetilde{Y}_u^{(t)}, \widetilde{Y}_v^{(t)} \rangle }
     { |\Omega_{uv}| }
=\frac{ \langle \widetilde{Y}_u^{(t)}, \widetilde{Y}_v^{(t)} \rangle_{\Omega_{uv}} }
      { \sqrt{\langle \widetilde{Y}_u^{(t)}, \widetilde{Y}_u^{(t)} \rangle_{\Omega_{uv}}}
        \sqrt{\langle \widetilde{Y}_v^{(t)}, \widetilde{Y}_v^{(t)} \rangle_{\Omega_{uv}}} }\,,
\]
which is the cosine similarity of revealed rating vectors $\widetilde{Y}_u^{(t)}$ and
$\widetilde{Y}_v^{(t)}$ restricted to the overlap of their supports. Thus, users $u$ and
$v$ are neighbors if and only if their cosine similarity is at least $\theta$.

\noindent\textbf{Theoretical performance guarantee.}
We now state our main result on the proposed collaborative filtering
algorithm's performance with respect to the objective stated in equation
\eqref{eq:reward+}. 
We begin with
two reasonable, and seemingly necessary, conditions under which our the 
results will be established.

\begin{itemize}

\item[\textbf{A1}] \textbf{No $\Delta$-ambiguous items.}
There exists some constant $\Delta > 0$ such that
\[
| p_{ui} - 1/2 | \ge \Delta
\]
for all users $u$ and items $i$. (Smaller $\Delta$ corresponds
to more noise.)

\item[\textbf{A2}] \textbf{$\gamma$-incoherence.}
There exist a constant $\gamma \in [0,1)$ such that if users $u$ and $v$ are
of different types, then their item preference vectors $p_u$ and $p_v$
satisfy
\[
\frac{1}{m}
\langle 2 p_u - \mathbf{1}, 2 p_v - \mathbf{1} \rangle
\le 4 \gamma \Delta^2,
\]
where $\mathbf{1}$ is the all ones vector. Note that a different way to write
the left-hand side is
$\E[\frac{1}{m} \langle Y_u^*, Y_v^* \rangle]$, where
$Y_u^*$ and $Y_v^*$ are fully-revealed 
rating vectors of users $u$ and
$v$, and the expectation is over the random ratings of
items.
\end{itemize}

The first condition is a low noise condition to ensure that with a finite
number of samples, we can correctly classify each item as either likable or
unlikable. The incoherence condition asks that the different user types are
well-separated so that cosine similarity can tease
apart the users of different types over time. 
We provide some examples after the statement of the main theorem that suggest the
incoherence condition to be reasonable, allowing
$\E[\langle Y_u^*, Y_v^* \rangle]$ to scale as $\Theta(m)$ rather than $o(m)$.

We assume that the number of users satisfies $n = O(m^C)$ for some
constant $C > 1$. This is without loss of generality since otherwise, we can
randomly divide the $n$ users into separate population pools, each of size
$O(m^C)$ and run the recommendation algorithm independently for each pool to
achieve the same overall performance guarantee.


Finally, we define $\mu$, the minimum proportion 
of likable items for any user (and thus any user type):
\[
\mu
\triangleq \min_{u \in [n]} \frac{\sum_{i=1}^m \ind\{ p_{ui} > 1/2 \}}{m}.
\]

\begin{theorem}
\label{thm:main}
Let $\delta\in(0,1)$ be some pre-specified tolerance. Take as input to
\textsc{Collaborative-Greedy} $\theta = 2 \Delta^2 (1 + \gamma)$ where
$\gamma \in [0, 1)$ is as defined in~\textbf{A2}, and $\alpha \in (0,4/7]$. Under the
latent source model and assumptions \textbf{A1} and \textbf{A2}, if the number
of users $n=O(m^C)$ satisfies
\begin{align*}
n
=
  \Omega
  \Big(km \log \frac{1}{\delta} +
       \Big(\frac{4}{\delta}\Big)^{1/\alpha}\Big), 
\end{align*}
%
%
%
then for any
$\tc \le T \le \mu m$, the expected proportion of likable items recommended by
\textsc{Collaborative-Greedy} up until time $T$ satisfies
\[
\frac{r_+^{(T)}}{T n}
\ge
  \Big(1 - \frac{\tc}{T}\Big)(1-\delta),
\]
where
\begin{align*}
\tc &=
\Theta
\bigg(
  \bigg(
    \frac{\log \frac{km}{\Delta \delta}}{\Delta^4 (1-\gamma)^2}
  \bigg)^{1/(1-\alpha)}
  + \Big(\frac{4}{\delta}\Big)^{1/\alpha}
\bigg).
%
%
%
\end{align*}
\end{theorem}

Theorem~\ref{thm:main} says that there are $\tc$ initial time steps for which
the algorithm may be giving poor recommendations. Afterward, for
$\tc<T<\mu m$, the algorithm becomes near-optimal, recommending a fraction of
likable items $1-\delta$ close to what an optimal oracle algorithm (that
recommends all likable items first) would achieve. Then for time horizon
$T>\mu m$, we can no longer guarantee that there are likable items left to
recommend. Indeed, if the user types each have the same fraction of likable
items, then even an oracle recommender would use up the $\mu m$ likable items
by this time.
Meanwhile, to give a sense of how long the learning period $\tc$ is, note that when $\alpha = 1/2$, we have
$\tc$ scaling as $\log^2(km)$, and if we choose $\alpha$ close to $0$, then
$\tc$ becomes nearly $\log (km)$. In summary,
after $\tc$ initial time steps, the simple algorithm proposed is essentially
optimal.



To gain intuition for incoherence condition \textbf{A2}, we calculate the
parameter $\gamma$ for three examples. 

\begin{example}
Consider when there is no noise, i.e., $\Delta = \frac12$. Then users' ratings
are deterministic given their user type. Produce $k$ vectors of
probabilities by drawing $m$ independent $\text{Bernoulli}(\frac12)$ random
variables ($0$ or $1$ with probability $\frac12$ each) for each user type. For any item
$i$ and pair of users $u$ and $v$ of different types,
$Y_{ui}^*\cdot Y_{vi}^*$ is a Rademacher random variable
($\pm 1$ with probability $\frac12$ each), and thus the inner product of two
user rating vectors is equal to the sum of $m$ Rademacher random variables.
Standard concentration inequalities show that one may take
$\gamma = \Theta\big(\sqrt{\frac{\log m}{m}}\big)$ to satisfy
$\gamma$-incoherence with probability $1-1/{\sf poly}(m)$. 
\end{example}

\begin{example}
We expand on the previous example by choosing an arbitrary $\Delta>0$ and making all latent source probability vectors have entries equal to $\frac12\pm \Delta$ with probability $\frac12$ each. As
before let user $u$ and $v$ are from different type. 
Now $\E[Y_{ui}^*\cdot Y_{vi}^*] = (\frac12 +\De)^2 + (\frac12 -\De)^2 - 2(\frac14 - \De^2)= 4\De^2$ if $p_{ui}=p_{vi}$ and $\E[Y_{ui}^*\cdot Y_{vi}^*] = 2(\frac14 -\De^2) - (\frac12+\Delta)^2 - (\frac12-\Delta)^2 = -4\De^2$ if $p_{ui}=1-p_{vi}$. The value of the inner product $\E[ \langle Y_u^*, Y_v^* \rangle]$ is again equal to the sum of $m$ Rademacher random variables, but this time scaled by $4\Delta^2$. For similar reasons as before, $\gamma = \Theta\big(\sqrt{\frac{\log m}{m}}\big)$ suffices 
to satisfy $\gamma$-incoherence with probability $1-1/{\sf poly}(m)$.
\end{example}

\begin{example}
Continuing with the previous example, now suppose each entry is $\frac12 + \Delta$ with probability $\mu \in (0,1/2)$ and $\frac12 - \Delta$ with probability $1-\mu$. Then for two users $u$ and $v$ of different
types, $p_{ui}=p_{vi}$ with probability $\mu^2 + (1-\mu)^2$. This implies that $\E[\langle Y_u^*, Y_v^* \rangle] = 4m \Delta^2 (1-2\mu)^2$.  Again, using standard concentration,   
this shows that $\gamma = (1-2\mu)^2 + \Theta\big(\sqrt{\frac{\log m}{m}}\big)$ suffices to satisfy $\gamma$-incoherence with probability $1-1/{\sf poly}(m)$.
\end{example}

\section{Proof of Theorem~\ref{thm:main}}\label{sec:proof}



Recall that $X_{ut}$ is the indicator random variable for whether the item
$\pi_{ut}$ recommended to user $u$ at time $t$ is likable, i.e.,
$p_{u \pi_{ut}} > 1/2$. Given assumption \textbf{A1}, this is equivalent to
the event that $p_{u \pi_{ut}} \ge \frac12 + \Delta$. The expected proportion
of likable items is
\begin{equation*}
\frac{r_+^{(T)}}{T n}
= \frac{1}{T n} \sum_{t=1}^T \sum_{u=1}^n \E[X_{ut}]
= \frac{1}{T n} \sum_{t=1}^T \sum_{u=1}^n \P( X_{ut}=1).
\end{equation*}
Our proof focuses on lower-bounding $\P(X_{ut}=1)$. The key idea is to
condition on what we call the ``good neighborhood'' event $\En(u,t)$:
\begin{align*}
\En(u,t) = \Big\{ &\text{ at time $t$, user
$u$ has $\ge \frac{n}{5k}$ neighbors from the same user type
(``good neighbors"),} \\
& \text{ and $\le \frac{\Delta t n^{1-\alpha}}{10km}$ neighbors from other
 user types
(``bad neighbors")}\Big\}.
\end{align*}
This good neighborhood event will enable us to argue that after an initial
learning time, with high probability there are at most $\Delta$ as many
ratings from bad neighbors as there are from good neighbors.

The proof of Theorem \ref{thm:main} consists of two parts. The first part uses
joint exploration to show that after a sufficient amount of time, the good
neighborhood event $\En(u,t)$ holds with high probability.

\begin{lemma}
\label{lem:En-happens}
For user $u$, after
%
%
\[
t
\ge
%
\bigg(\frac{2 \log (10kmn^\alpha/\Delta)}{\Delta^4(1-\gamma)^2}\bigg)^{1/(1-\alpha)}
\]
time steps,
\begin{align*}
\P(\En(u,t)) 
&\ge
   1 - \exp\Big(-\frac{n}{8k}\Big)
     - 12\exp\Big(-\frac{\Delta^4 (1-\gamma)^2 t^{1-\alpha}}{20} \Big)\,.
\end{align*}
\end{lemma}

In the above lower bound, the first exponentially decaying term could be
thought of as the penalty for not having enough users in the system from the
$k$ user types, and the second decaying term could be thought of as the
penalty for not yet clustering the users correctly.

The second part of our proof to Theorem \ref{thm:main} shows that, with
high probability, the good neighborhoods have, through random exploration,
accurately estimated the probability of liking each item. Thus, we correctly
classify each item as likable or not with high probability, which leads to a
lower bound on $\P(X_{ut}=1)$.

\begin{lemma}
\label{lem:final}
For user $u$ at time $t$, if the good neighborhood event $\En(u,t)$ holds
and $t \le \mu m$, then
\begin{align*}
\P(X_{ut}=1) 
&\ge
   1
   - 2m\exp\Big(-\frac{\Delta^2 t n^{1-\alpha}}{40km}\Big)
   - \frac{1}{t^\alpha} - \frac{1}{n^\alpha}\,.
\end{align*}
\end{lemma}

Here, the first exponentially decaying term could be thought of as the cost of
not classifying items correctly as likable or unlikable, and the last two
decaying terms together could be thought of as the cost of exploration (we
explore with probability
$\varepsilon_J(t) + \varepsilon_R(n) = 1/t^\alpha + 1/n^\alpha$).

We defer the proofs of Lemmas~\ref{lem:En-happens} and \ref{lem:final} to
Appendices \ref{sec:appendix-proof-En-happens} and \ref{appendix:lemmafinal}.
Combining these lemmas and choosing
appropriate constraints on the numbers of users and items, we produce the
following lemma.

\begin{lemma}
\label{lem:general-conditions}
Let $\delta \in (0,1)$ be some pre-specified tolerance.
If the number of users $n$ and items $m$ satisfy
\begin{align*}
n  & \ge \max\Big\{
               8k\log\frac{4}{\delta},
               \Big(\frac{4}{\delta}\Big)^{1/\alpha}
             \Big\}, \\
\mu m \ge
t  & \ge \max\bigg\{
 \bigg(\frac{2 \log (10kmn^\alpha/\Delta)}{\Delta^4(1-\gamma)^2}\bigg)^{1/(1-\alpha)}, 
%
                \bigg(
                \frac{20\log(96/\delta)}{\Delta^4 (1-\gamma)^2}
                \bigg)^{1/(1-\alpha)},
                \Big(\frac{4}{\delta}\Big)^{1/\alpha}
             \bigg\}, \\
n t^{1-\alpha}
   & \ge \frac{40 k m}{\Delta^2} \log\Big(\frac{16m}{\delta}\Big),
\end{align*}
then
$
\P(X_{ut}=1) \ge 1 - \delta.
$
\end{lemma}

\begin{proof}
With the above conditions on $n$ and $t$ satisfied, we combine Lemmas
\ref{lem:En-happens} and \ref{lem:final} to obtain
\begin{align*}
\P(X_{ut}=1) 
&\ge
   1 - \exp\Big(-\frac{n}{8k}\Big)
     - 12\exp\Big(-\frac{\Delta^4 (1-\gamma)^2 t^{1-\alpha}}{20} \Big)
     - 2m\exp\Big(-\frac{\Delta^2 t n^{1-\alpha}}{40km}\Big) \\
&\quad
   - \frac{1}{t^\alpha} - \frac{1}{n^\alpha} 
\ge
  1 - \frac{\delta}{4}
  - \frac{\delta}{8}
  - \frac{\delta}{8}
  - \frac{\delta}{4}
  - \frac{\delta}{4}
= 1 - \delta. \qedhere
\end{align*}

\end{proof}

Theorem \ref{thm:main} follows as a corollary to Lemma
\ref{lem:general-conditions}. As previously mentioned, without loss of
generality, we take $n=O(m^C)$. Then with number of users $n$ satisfying
\begin{equation*}
O(m^C)
=
n
=
  \Omega
  \Big(km \log \frac{1}{\delta} +
             \Big(\frac{4}{\delta}\Big)^{1/\alpha}\Big), 
\end{equation*}
and for any time step $t$ satisfying
\begin{equation*}
\mu m
\ge
t
\ge
  \Theta
  \bigg(
    \bigg(
      \frac{\log \frac{km}{\Delta \delta}}{\Delta^4 (1-\gamma)^2}
    \bigg)^{1/(1-\alpha)}
    + \Big(\frac{4}{\delta}\Big)^{1/\alpha}
  \bigg)
\triangleq \tc\,, 
\end{equation*}
we simultaneously meet all of the conditions of
Lemma~\ref{lem:general-conditions}. Note that the upper bound on number of
users $n$ appears since without it, $\tc$ would depend on $n$ (observe that in
Lemma \ref{lem:general-conditions}, we ask that $t$ be greater than a quantity
that depends on $n$).
Provided that the time horizon satisfies
$T \le \mu m$, then 
\begin{align*}
\frac{r_+^{(T)}}{T n}
&\ge
   \frac{1}{T n} \sum_{t=\tc}^T \sum_{u=1}^n \P(X_{ut}=1) 
\ge
  \frac{1}{T n} \sum_{t=\tc}^T \sum_{u=1}^n (1-\delta) 
= \frac{(T - \tc)(1-\delta)}{T},
\end{align*}
yielding the theorem statement.


\section{Experimental Results}
\label{sec:experiments}

We provide only a summary of our experimental results here, deferring
full details to Appendix \ref{sec:appendix-experimental-results}.
We simulate an online recommendation system based on movie ratings from the
Movielens10m and Netflix datasets, each of which provides a sparsely filled
user-by-movie rating matrix with ratings out of 5 stars. Unfortunately,
existing collaborative filtering datasets such as the two we consider don't
offer the interactivity of a real online recommendation system, nor do they
allow us to reveal the rating for an item that a user didn't actually rate.
For simulating an online system, the former issue can be dealt with by
simply revealing entries in the user-by-item rating matrix over time. We
address the latter issue by only considering a dense ``top users vs.~top
items'' subset of each dataset. In particular, we consider only the ``top''
users who have rated the most number of items, and the ``top'' items that have
received the most number of ratings. While this dense part of the dataset is
unrepresentative of the rest of the dataset, it does allow us to use actual
ratings provided by users without synthesizing any ratings. A rigorous
validation would require an implementation of an actual interactive online
recommendation system, which is beyond the scope of our paper.

First, we validate that our latent source model is reasonable for the dense
parts of the two datasets we consider by looking for clustering behavior
across users. We find that the dense top users vs.~top movies matrices do in
fact exhibit clustering behavior of users and also movies, as shown in Figure
\ref{fig:experimental-results}\subref{fig:bctf}. The clustering was found via
Bayesian clustered tensor factorization, which was previously shown to model
real movie ratings data well \cite{bctf_2009}.

Next, we demonstrate our algorithm
\textsc{Collaborative-Greedy} on the two simulated online movie recommendation
systems, showing that it outperforms two existing recommendation algorithms
Popularity Amongst Friends (PAF) \cite{paf} and a method by Deshpande and
Montanari (DM) \cite{deshpande_montanari}.
Following the experimental setup of \cite{paf},
we quantize a rating of 4 stars or more to be $+1$ (likable), and a rating
less than 4 stars to be $-1$ (unlikable). While we look at a dense subset
of each dataset, there are still missing entries. If a user $u$ hasn't rated
item $j$ in the dataset, then we set the corresponding true rating to 0,
meaning that in our simulation, upon recommending item $j$ to user $u$, we
receive 0 reward, but we still mark that user $u$ has consumed item $j$; thus,
item $j$ can no longer be recommended to user $u$. For both Movielens10m and
Netflix datasets, we consider the top $n=200$ users and the top $m=500$
movies. For Movielens10m, the resulting user-by-rating matrix has 80.7\%
nonzero entries. For Netflix, the resulting matrix has 86.0\% nonzero entries.
For an algorithm that recommends item $\pi_{ut}$ to user $u$ at time $t$, we
measure the algorithm's average cumulative reward up to time $T$ as
$
\frac{1}{n} \sum_{t=1}^T\sum_{u=1}^n Y_{u \pi_{ut}}^{(T)},
$
where we average over users.
For all four methods, we recommend items until we reach time $T=500$, 
i.e., we make movie recommendations until each user
has seen all $m=500$ movies. 
We disallow the matrix completion step for DM to see the users that we
actually test on, but we allow it to see the the same items as what is in the
simulated online recommendation system in order to compute these items'
feature vectors (using the rest of the users in the dataset). Furthermore,
when a rating is revealed, we provide DM both the thresholded rating and the
non-thresholded rating, the latter of which DM uses to estimate user feature
vectors over time. We discuss choice of algorithm parameters in
Appendix \ref{sec:appendix-experimental-results}.
In short, parameters $\theta$ and $\alpha$ of our
algorithm are chosen based on training data, whereas we allow the other
algorithms to use whichever parameters give the best results on the test data.
Despite giving the two competing algorithms this advantage,
\textsc{Collaborative-Greedy} outperforms the two, as shown in Figure
\ref{fig:experimental-results}\subref{fig:results-top}. Results on the Netflix
dataset are similar.

\begin{figure}[t]
\captionsetup[subfloat]{farskip=0em,captionskip=-.4em,nearskip=-5em}
\noindent
\centering
\subfloat[]{
\includegraphics[width=6.7cm]{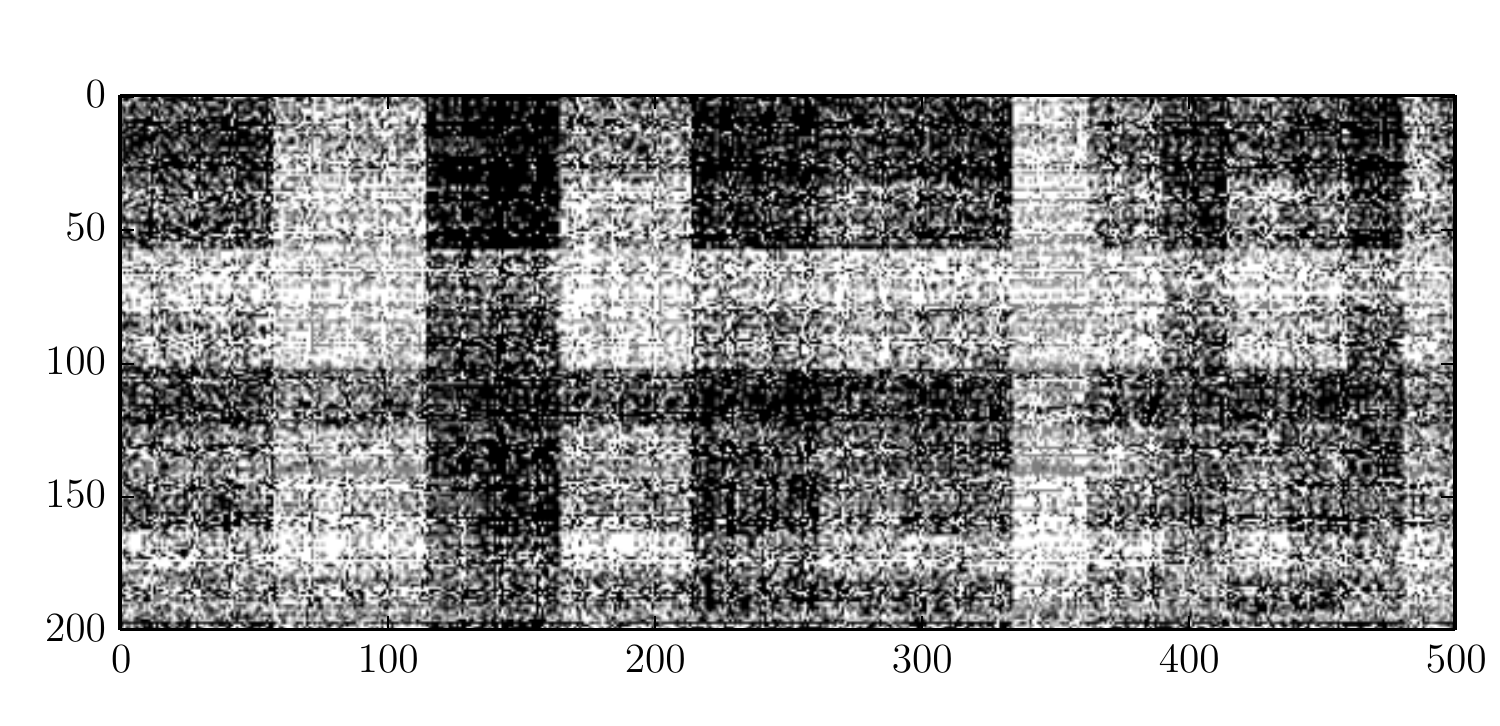}
\label{fig:bctf}
}
\subfloat[]{
\includegraphics[width=6.7cm]{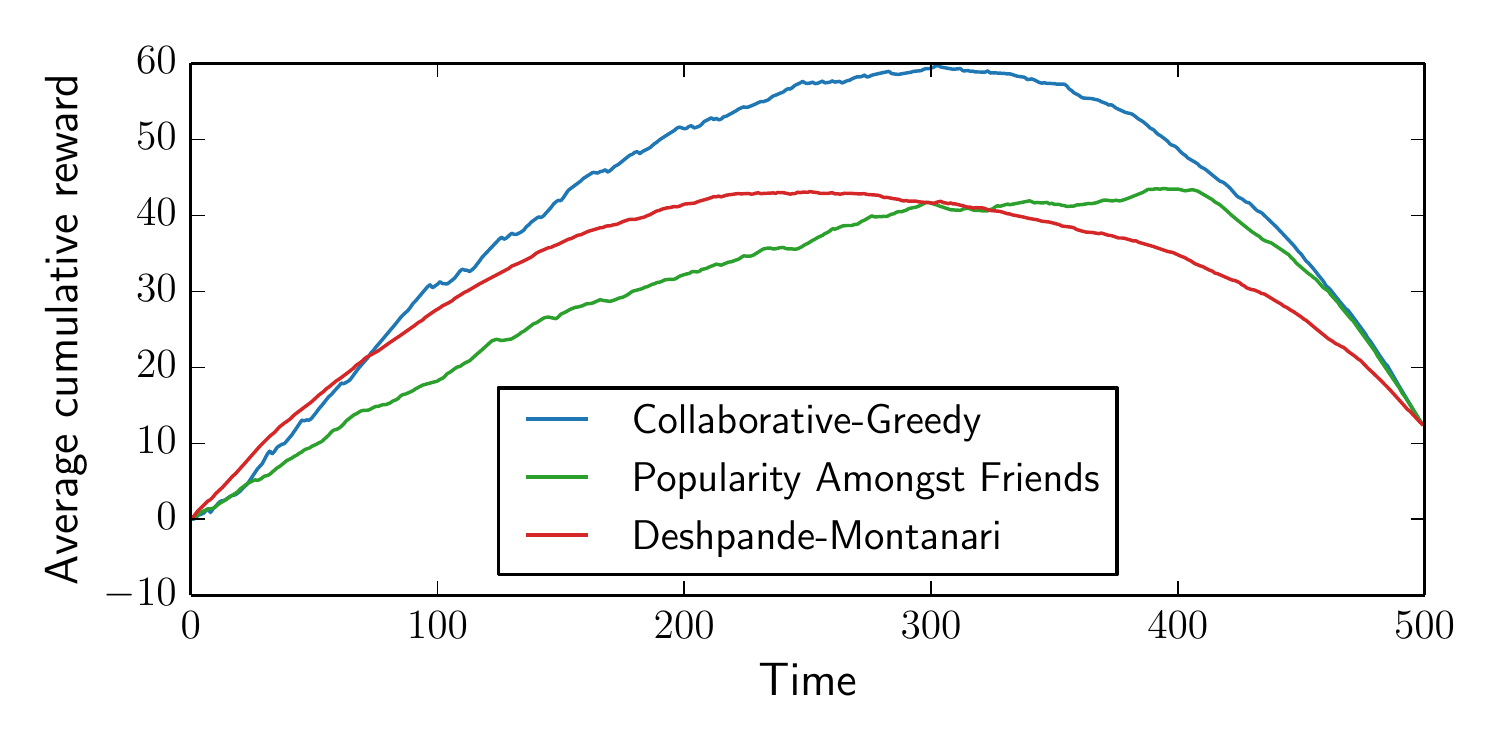}
\label{fig:results-top}
}
\caption{Movielens10m dataset: (a) Top users by top movies matrix with rows and
columns reordered to show clustering of users and items. (b) Average cumulative
rewards over time.
\vspace{-.8em}
}
\label{fig:experimental-results}
\end{figure}

\section{Discussion and Related Work}
\label{sec:discussion}

This paper proposes a model for online recommendation systems under which we
can analyze the performance of recommendation algorithms. We theoretical
justify when a cosine-similarity collaborative filtering method works well,
with a key insight of using two exploration types. 

The closest related work is by Biau et al.~\cite{statistical_knn_cf}, who
study the asymptotic consistency of a cosine-similarity nearest-neighbor
collaborative filtering method. Their goal is to predict the rating of the
next unseen item. Barman and Dabeer \cite{paf} study the performance of an
algorithm called Popularity Amongst Friends, examining its ability to predict
binary ratings in an asymptotic information-theoretic setting. In contrast, we
seek to understand the finite-time performance of such systems. Dabeer
\cite{dabeer_online} uses a model similar to ours and studies online
collaborative filtering with a moving horizon cost in the limit of small noise
using an algorithm that knows the numbers of user types and item types. We do
not model different item types, our algorithm is oblivious to the
number of user types, and our performance metric is different. Another related
work is by Deshpande and Montanari \cite{deshpande_montanari}, who study
online recommendations as a linear bandit problem; their method, however, does
not actually use any collaboration beyond a pre-processing step in which
offline collaborative filtering (specifically matrix completion) is solved to
compute feature vectors for items.

Our work also relates to the problem of
learning mixture distributions (c.f.,
\cite{chaudhuri2008learning,mixture0, belkin2010polynomial, mixture1}),
where one observes samples from a mixture distribution and the goal 
is to learn the mixture components and weights. Existing results assume 
that one has access to the entire high-dimensional sample or that the samples 
are produced in an exogenous manner (not chosen by the algorithm). Neither
assumption holds in our setting, as we only see each user's
revealed ratings thus far and not the user's entire preference vector, and the
recommendation algorithm affects which samples are observed (by choosing which
item ratings are revealed for each user). These two aspects make our setting
more challenging than the standard setting for learning
mixture distributions. However, our goal is more modest. Rather than learning
the $k$ item preference vectors, we settle for classifying them as likable or
unlikable. Despite this, we suspect having two types of exploration to be
useful in general for efficiently learning mixture distributions in the active
learning setting.

\textbf{Acknowledgements.}
This work was supported in part by NSF grant CNS-1161964 and by Army
Research Office MURI Award W911NF-11-1-0036.
GHC was supported by an NDSEG fellowship.


\small
\bibliography{cf_latent_source}
\normalsize

\clearpage
\appendix

\section{Appendix}

Throughout our derivations, if it is clear from context, we omit the argument
$(t)$ indexing time, for example writing $Y_u$ instead of $Y_u(t)$.

\subsection{Proof of Lemma~\ref{lem:En-happens}}
\label{sec:appendix-proof-En-happens}

We reproduce Lemma~\ref{lem:En-happens} below for ease of presentation.

\newtheorem*{lem-En-happens}{Lemma \ref{lem:En-happens}}
\begin{lem-En-happens}
For user $u$, after
%
%
\[
t
\ge
%
\bigg(\frac{2 \log (10kmn^\alpha/\Delta)}{\Delta^4(1-\gamma)^2}\bigg)^{1/(1-\alpha)}
\]
time steps,
\begin{align*}
\P(\En(u,t)) 
&\ge
   1 - \exp\Big(-\frac{n}{8k}\Big)
     - 12\exp\Big(-\frac{\Delta^4 (1-\gamma)^2 t^{1-\alpha}}{20} \Big)\,.
\end{align*}
\end{lem-En-happens}

To derive this lower bound on the probability that the good neighborhood event
$\En(u,t)$ occurs, we prove four lemmas (Lemmas
\ref{lem:latent-sources-coupon-collector},
\ref{lem:joint-exploration1},
\ref{lem:num-good-neighbors-bound}, and
\ref{lem:num-bad-neighbors-bound}).
Before doing so, we define a
constant that will appear several times:
\begin{equation*}
\beta \triangleq
\exp(-\Delta^4 (1-\gamma)^2 t^{1-\alpha} ).
\end{equation*}
We begin by ensuring that enough users from each of the $k$ user types are in the system.


\begin{lemma}
\label{lem:latent-sources-coupon-collector}
For a user $u$,
\begin{equation*}
\P\Big(
    \text{user $u$'s type has}\le\frac{n}{2k}
    \text{ users}
  \Big)
\le
  \exp\Big(-\frac{n}{8k}\Big).
\label{eq:collecting-many-coupons-main-bound}
\end{equation*}
\end{lemma}

\begin{proof}
Let $N$ be the number of users from user $u$'s type. User types are
equiprobable, so $N \sim \text{Bin}(n, \frac{1}{k})$. 
By a Chernoff bound,
\begin{equation*}
\P\Big(N \le \frac{n}{2k} \Big)
\le\exp\bigg(
         -\frac{1}{2}
          \frac{(\frac{n}{k}-\frac{n}{2k})^2}
               {\frac{n}{k}}
       \bigg)
=\exp\Big(-\frac{n}{8k}\Big).\qedhere
\end{equation*}
\end{proof}

Next, we ensure that sufficiently many items have been jointly explored across all users. This will subsequently be used for bounding both the number of good neighbors and the number
of bad neighbors.

\begin{lemma}
\label{lem:joint-exploration1}
After $t$ time steps,
\begin{align*}
\P(\text{fewer than }t^{1-\alpha}/2\text{ jointly explored items}) 
&\le\exp(-t^{1-\alpha}/20).
\end{align*}
\end{lemma}
\begin{proof}
Let $Z_s$ be the indicator random variable for the event that the algorithm
jointly explores at time~$s$. Thus, the number of jointly explored items up to
time $t$ is $\sum_{s=1}^t Z_s$. By our choice for the time-varying joint
exploration probability $\varepsilon_J$, we have
$\P(Z_s=1)=\varepsilon_J(s)=\frac{1}{s^\alpha}$ and
$\P(Z_s=0) = 1-\frac{1}{s^\alpha}$. Note that the centered random variable
$\bar{Z}_s =
 \E[Z_s] - Z_s =
 \frac{1}{s^\alpha} - Z_s$ has zero mean, and $|\bar{Z}_s| \le 1$
with probability 1. Then,
\begin{align*}
\P\bigg( \sum_{s=1}^t Z_s < \frac{1}{2}t^{1-\alpha} \bigg)
&=\P\bigg(
      \sum_{s=1}^t \bar{Z}_s > \sum_{s=1}^t \E[Z_s] - \frac{1}{2}t^{1-\alpha}
    \bigg) 
\overset{(i)}{\le}  \P\bigg( \sum_{s=1}^t \bar{Z}_s > \frac{1}{2}t^{1-\alpha} \bigg) \\
&\overset{(ii)}{\le}
  \exp\bigg(
        -\frac{\frac{1}{8}t^{2 (1-\alpha)}}
              {\sum_{s=1}^t \E[\bar{Z}_s^2] +
               \frac{1}{6}t^{1-\alpha}}
      \bigg) 
\overset{(iii)}{\le}  \exp\bigg(
        -\frac{\frac{1}{8}t^{2 (1-\alpha)}}
              {\frac{t^{1-\alpha}}{1-\alpha} +
               \frac{1}{6}t^{1-\alpha}}
      \bigg) \\
&=
  \exp\bigg(
        -\frac{3(1-\alpha)t^{1-\alpha}}
              {4(7-\alpha)}
      \bigg) 
\overset{(iv)}{\le}  \exp(-t^{1-\alpha}/20),
\end{align*}
where step $(i)$ uses the fact that
$\sum_{s=1}^t \E[Z_s]
 = \sum_{s=1}^t 1/s^\alpha
 \ge t/t^\alpha
 = t^{1-\alpha}$,
step $(ii)$ is Bernstein's inequality, step $(iii)$ uses the fact that
$\sum_{s=1}^t \E[\bar{Z}_s^2]
 \le \sum_{s=1}^t \E[Z_s^2]
 = \sum_{s=1}^t 1/s^\alpha
 \le t^{1-\alpha}/(1-\alpha)$,
and step $(iv)$ uses the fact that $\alpha\le4/7$.
(We remark that the choice of constant $4/7$ isn't special; changing it would
simply modify the constant in the decaying exponentially to potentially no
longer be $1/20$).
\end{proof}

Assuming that the bad events for the previous two lemmas do not occur, we now
provide a lower bound on the number of good neighbors that holds with high
probability.

\begin{lemma}
\label{lem:num-good-neighbors-bound}
Suppose that there are no $\Delta$-ambiguous items, that there are more than
$\frac{n}{2k}$ users of user $u$'s type, and that all users have rated at
least $t^{1-\alpha}/2$ items as part of joint exploration. For user $u$, let $n_{\text{good}}$ be the
number of ``good'' neighbors of user $u$.  If $\beta \le \frac{1}{10}$, then
\[
\P\Big( n_{\text{good}}\leq (1-\beta)\frac{n}{4k} \Big)
\le 10\beta.
\]
\end{lemma}

We defer the proof of Lemma~\ref{lem:num-good-neighbors-bound} to
Appendix~\ref{sec:proof-num-good-neighbors-bound}.

Finally, we verify that the number of bad neighbors for any user is not too
large, again conditioned on there being enough jointly explored items.

\begin{lemma}
\label{lem:num-bad-neighbors-bound}
Suppose that the minimum number of rated items in common between any pair of
users is $t^{1-\alpha}/2$ and suppose that $\gamma$-incoherence holds
for some $\gamma \in [0,1)$. For user $u$,
let $n_{\text{bad}}$ be the number of ``bad'' neighbors of user $u$. 
Then
\begin{equation*}
\P(n_\text{bad} \ge n\sqrt{\beta}) \le \sqrt{\beta}.
\end{equation*}
\end{lemma}

We defer the proof of Lemma~\ref{lem:num-bad-neighbors-bound} to
Appendix~\ref{sec:proof-num-bad-neighbors-bound}.

We now prove Lemma~\ref{lem:En-happens}, which union bounds over the four bad
events of Lemmas
\ref{lem:latent-sources-coupon-collector},
\ref{lem:joint-exploration1},
\ref{lem:num-good-neighbors-bound}, and
\ref{lem:num-bad-neighbors-bound}.
Recall that the good neighborhood event
$\En(u,t)$ holds if at time $t$, user $u$ has more than $\frac{n}{5k}$
good neighbors 
and less than
$\frac{\Delta t n^{1-\alpha}}{10km}$ bad neighbors.
By assuming that the four bad events don't happen, then Lemma
\ref{lem:num-good-neighbors-bound} tells us that there are more than
$(1-\beta)\frac{n}{4k}$ good neighbors provided that $\beta \le \frac{1}{10}$.
Thus, to ensure that there are
more than $\frac{n}{5k}$ good neighbors, it suffices to have
$(1-\beta)\frac{n}{4k} \ge \frac{n}{5k}$, which happens when
$\beta \le \frac{1}{5}$, but we already require that $\beta \le \frac{1}{10}$.
Similarly, Lemma \ref{lem:num-bad-neighbors-bound} tells us that there are
fewer than $n\sqrt{\beta}$ bad neighbors, so to ensure that there are fewer
than
{$\frac{\Delta t n^{1-\alpha}}{10km}$}
bad neighbors it suffices to have
{$n\sqrt{\beta} \le \frac{\Delta t n^{1-\alpha}}{10km}$}, which happens when
$ \beta \le ( \frac{\Delta t}{10 k m n^\alpha} )^2 $.
We can satisfy all constraints on $\beta$ by asking that
$ \beta \le ( \frac{\Delta}{10 k m n^\alpha} )^2 $, which is tantamount to
asking that
\[
t
\ge
%
\bigg(\frac{2 \log (10kmn^\alpha/\Delta)}{\Delta^4(1-\gamma)^2}\bigg)^{1/(1-\alpha)}
\]
since
$\beta = \exp(-\Delta^4 (1-\gamma)^2 t^{1-\alpha})$.

Finally, with $t$ satisfying the inequality above, the union bound over the
four bad events can be further bounded to complete the proof:
\begin{align*}
\P(\En(u,t))
&\ge
   1 - \exp\Big(-\frac{n}{8k}\Big) - \exp(-t^{1-\alpha}/20)
   - 10\beta - \sqrt{\beta} \\
&\ge
   1 - \exp\Big(-\frac{n}{8k}\Big)
   - 12\exp\Big(-\frac{\Delta^4(1-\gamma)^2 t^{1-\alpha}}{20}\Big).
\end{align*}
\subsubsection{Proof of Lemma~\ref{lem:num-good-neighbors-bound}}
\label{sec:proof-num-good-neighbors-bound}

We begin with a preliminary lemma that upper-bounds the probability of two
users of the same type not being declared as neighbors.

\begin{lemma}
\label{lem:same-source-pairwise-error}
Suppose that there are no $\Delta$-ambiguous items for any of the user types.
Let users $u$ and $v$ be of the same type, and suppose that they have rated at
least $\Gamma_0$ items in common (explored jointly). Then for
$\theta \in (0, 4\Delta^2)$,
\begin{equation*}
\P(\text{users }u\text{ and }v\text{ are not declared as neighbors})
\le
  \exp\Big( -\frac{(4\Delta^2-\theta)^2}{2} \Gamma_0 \Big).
\end{equation*}
\end{lemma}

\begin{proof}

Let us first suppose that users $u$ and $v$ have rated exactly $\Gamma_0$ items
in common. The two users are not declared to be neighbors if
$\langle \widetilde{Y}_u, \widetilde{Y}_v \rangle < \theta\Gamma_0$. Let
$\Omega \subseteq [m]$ such
that $|\Omega|=\Gamma_0$. We have
\begin{align}
\E\big[
    \langle \widetilde{Y}_u, \widetilde{Y}_v \rangle
    \big|
    \text{supp}(\widetilde{Y}_u)\cap\text{supp}(\widetilde{Y}_v)=\Omega
  \big]
&
 =\sum_{i \in \Omega}
    \E[ \widetilde{Y}_{ui} \widetilde{Y}_{vi}
        \mid \widetilde{Y}_{ui}\ne0, \widetilde{Y}_{vi}\ne0]
 \nonumber \\
&
 =\sum_{i \in \Omega}
    (p_{ui}^2
     + (1-p_{ui})^2
     - 2 p_{ui} (1 - p_{ui}))
 \nonumber \\
&
 =4\sum_{i \in \Omega}
     \Big( p_{ui} - \frac{1}{2} \Big)^2.
\label{eq:same-source-E-inner-prod}
\end{align}
Since $\langle \widetilde{Y}_u, \widetilde{Y}_v \rangle
= \sum_{i \in \Omega} \widetilde{Y}_{ui} \widetilde{Y}_{vi}$ is the
sum of terms $\{\widetilde{Y}_{ui} \widetilde{Y}_{vi}\}_{i \in \Omega}$ that
are each bounded within $[-1,1]$, Hoeffding's inequality yields
\begin{equation}
\P\big(
    \langle \widetilde{Y}_u, \widetilde{Y}_v \rangle
    \le
      \theta\Gamma_0
    \;\big|\;
    \text{supp}(\widetilde{Y}_u)\cap\text{supp}(\widetilde{Y}_v)=\Omega
  \big) 
\le
  \exp\bigg(
        -\frac{\big[
                 \overbrace{4\textstyle{\sum_{i \in \Omega}}
                               \big(p_{gi}-\frac{1}{2}\big)^2}^{
                   \text{equation\,}\eqref{eq:same-source-E-inner-prod}}
                 -
                 \theta\Gamma_0
               \big]^2}
              {2\Gamma_0}
      \bigg).
\label{eq:same-source-inner-prod-hoeffding}
\end{equation}
As there are no $\Delta$-ambiguous items,
$\Delta \le |p_{ui} - 1/2|$ for all users $u$ and items $i$.
Thus, our choice of $\theta$ guarantees that
\begin{equation}
4\sum_{i \in \Omega}
   \Big( p_{ui} - \frac{1}{2} \Big)^2
- \theta\Gamma_0
\ge
  4\Gamma_0\Delta^2
  - \theta\Gamma_0
=
  (4\Delta^2 - \theta)\Gamma_0
\ge
  0.
\label{eq:same-source-hoeffding-var}
\end{equation}
Combining inequalities \eqref{eq:same-source-inner-prod-hoeffding} and
\eqref{eq:same-source-hoeffding-var}, and observing that the above holds for
all subsets $\Omega$ of cardinality $\Gamma_0$, we obtain the desired bound on the
probability that users $u$ and $v$ are not declared as neighbors:
\begin{equation}
\P( \langle \widetilde{Y}_u, \widetilde{Y}_v \rangle \le \theta\Gamma_0
   \;|\;
   |\text{supp}(\widetilde{Y}_u)\cap\text{supp}(\widetilde{Y}_v)|=\Gamma_0)
\le
  \exp\Big( -\frac{(4\Delta^2-\theta)^2}{2}\Gamma_0 \Big).
\label{eq:same-source-prob-error-exact-overlap-amount}
\end{equation}
Now to handle the case that users $u$ and $v$ have jointly rated more than
$\Gamma_0$ items, observe that, with shorthand
$\Gamma_{uv} \triangleq
 |\text{supp}(\widetilde{Y}_u)\cap\text{supp}(\widetilde{Y}_v)|$,
\begin{align*}
&\P(u\text{ and }v\text{ not declared neighbors}
    \,|\,
    p_u=p_v,
    \Gamma_{uv} \ge \Gamma_0) \nonumber \\
&=\P( \langle \widetilde{Y}_u, \widetilde{Y}_v \rangle < \theta\Gamma_{uv}
      \;|\;
      p_u=p_v,
      \;\Gamma_{uv} \ge \Gamma_0) \nonumber \\
&=\frac{\P(\langle \widetilde{Y}_u, \widetilde{Y}_v \rangle \le \theta\Gamma_{uv},
           \Gamma_{uv} \ge \Gamma_0
           \;|\;
           p_u=p_v)}
       {\P(\Gamma_{uv} \ge \Gamma_0
           \;|\; p_u=p_v)} \nonumber \\
&=\frac{\sum_{\ell=\Gamma_0}^m
          \P(\langle \widetilde{Y}_u, \widetilde{Y}_v \rangle \le\theta\ell,
             \Gamma_{uv}=\ell
             \;|\;
             p_u=p_v)}
       {\P(\Gamma_{uv} \ge \Gamma_0
           \;|\;
           p_u=p_v)} \nonumber \\
&=\frac{\sum_{\ell=\Gamma_0}^m
          \begin{array}{c}
          \big[\P(\Gamma_{uv}=\ell
                  \;|\;
                  p_u=p_v)\qquad\qquad\qquad\qquad\\
          \cdot
          \P(\langle \widetilde{Y}_u, \widetilde{Y}_v \rangle \le \theta\ell
                     \;|\;
                     p_u=p_v,
                     \Gamma_{uv}=\ell)\big]
          \end{array}}
       {\P(\Gamma_{uv} \ge \Gamma_0
           \;|\;
           p_u=p_v)} \nonumber \\
&\le
   \frac{\sum_{\ell=\Gamma_0}^m
           \P(\Gamma_{uv}=\ell
              \;|\;
              p_u=p_v)
           \exp\big( -\frac{(4\Delta^2-\theta)^2}{2}\Gamma_0 \big)}
        {\P(\Gamma_{uv} \ge \Gamma_0
            \;|\;
            p_u=p_v)} \nonumber \\
&\quad\;\;\text{by inequality }
 \eqref{eq:same-source-prob-error-exact-overlap-amount}
 \nonumber \\
&=
   \exp\Big( -\frac{(4\Delta^2-\theta)^2}{2} \Gamma_0 \Big). \qedhere
\end{align*}
\end{proof}

We now prove Lemma~\ref{lem:num-good-neighbors-bound}.

Suppose that the event in Lemma~\ref{lem:latent-sources-coupon-collector}
holds. Let $\mathcal{G}$ be $\frac{n}{2k}$ users from the same user type as
user $u$; there could be more than $\frac{n}{2k}$ such users but it suffices to consider $\frac{n}{2k}$ of them. We define an indicator random variable
\[
G_v
\triangleq
  \ind\{\text{users }u\text{ and }v\text{ are neighbors}\}
=\ind\{\langle \widetilde{Y}_u^{(t)}, \widetilde{Y}_v^{(t)} \rangle
       \ge \theta t^{1-\alpha}/2\}.
\]
Thus, the number of good neighbors of user $u$ is lower-bounded by
$W=\sum_{v \in \mathcal{G}} G_v$. Note that the $G_v$'s are not independent.
To arrive at a lower bound for $W$ that holds with high probability, we use
Chebyshev's inequality:
\begin{equation}
\label{e:Cheby}
\P(W-\E[W]\le - \E[W]/2)\leq\frac{4\var(W)}{(\E[W])^2}\,.
\end{equation}
Let $\beta = \exp(-(4\De^2-\theta)^2 \Gamma_0/2)$ be the probability bound
from Lemma~\ref{lem:same-source-pairwise-error}, where by our choice of
$\theta = 2 \Delta^2 (1 + \gamma)$ and with $\Gamma_0=t^{1-\alpha}/2$, we have
$\beta = \exp(-\Delta^4(1-\gamma)^2 t^{1-\alpha})$.

Applying Lemma~\ref{lem:same-source-pairwise-error}, we have
$\E[W]\geq (1-\beta) \frac{n}{2k}$, and hence
\begin{equation}
\label{eq:Cheby-squared-Exp-W}
(\E[W])^2 \ge (1-2\beta) \frac{n^2}{4k^2}.
\end{equation}
We now upper-bound
\[
\text{Var}(W)
=\sum_{v \in \mathcal{G}}
   \text{Var}(G_v) +
 \sum_{v \ne w}
   \text{Cov}(G_v, G_w).
\]

Since $G_v=G_v^2$,
\[
\text{Var}(G_v)
=\E[G_v]-\E[G_v]^2
=\underbrace{\E[G_v]}_{\le 1}
 (1-\E[G_v])
\le \beta,
\]
where the last step uses Lemma~\ref{lem:same-source-pairwise-error}.

Meanwhile,
\[
\text{Cov}(G_v, G_w)
=\E[G_v G_w]-\E[G_v]\E[G_w]
\le 1-(1-\beta)^2
\le 2 \beta.
\]
Putting together the pieces,
\begin{equation}
\label{eq:Cheby-var-W}
\text{Var}(W)
\le
  \frac{n}{2k} \cdot \beta +
  \frac{n}{2k} \cdot
    \Big( \frac{n}{2k}-1 \Big) \cdot 2\beta
\le
  \frac{n^2}{2 k^2} \cdot \beta.
\end{equation}
Plugging \eqref{eq:Cheby-squared-Exp-W} and \eqref{eq:Cheby-var-W} into
\eqref{e:Cheby} gives
\begin{equation*}
\P(W-\E[W]\le - \E[W]/2) \le \frac{8 \beta}{1-2\beta} \le 10\beta,
\end{equation*}
provided that $\beta \le \frac{1}{10}$. Thus,
$n_{\text{good}} \ge W \ge \E[W]/2 \ge (1-\beta)\frac{n}{4k}$ with probability
at least $1 - 10\beta$.

\subsubsection{Proof of Lemma~\ref{lem:num-bad-neighbors-bound}}
\label{sec:proof-num-bad-neighbors-bound}

We begin with a preliminary lemma that upper-bounds the probability of two
users of different types being declared as neighbors.

\begin{lemma}
\label{lem:diff-source-pairwise-error}
Let users $u$ and $v$ be of different types, and suppose that they have rated
at least $\Gamma_0$ items in common via joint exploration. Further suppose
$\gamma$-incoherence is satisfied for $\gamma \in [0,1)$. 
If
$\theta \ge 4\gamma\Delta^2$, then
\[
\P(\text{users }u\text{ and }v\text{ are declared to be neighbors})
\le
  \exp\Big( -\frac{(\theta-4\gamma\Delta^2)^2}{2}\Gamma_0 \Big).
\]
\end{lemma}
\begin{proof}
As with the proof of Lemma \ref{lem:same-source-pairwise-error}, we first
analyze the case where users $u$ and $v$ have rated exactly $\Gamma_0$ items
in common. Users $u$ and $v$ are declared to be neighbors if
$\langle \widetilde{Y}_u, \widetilde{Y}_v \rangle
 \ge \theta\Gamma_0$. We now crucially use the fact that joint exploration
chooses these $\Gamma_0$ items as a random subset of the $m$ items. For
our random permutation $\sigma$ of $m$ items, we have
$\langle \widetilde{Y}_u, \widetilde{Y}_v \rangle
 =\sum_{i=1}^{\Gamma_0}
    \widetilde{Y}_{u,\sigma(i)} \widetilde{Y}_{v,\sigma(i)}
 =\sum_{i=1}^{\Gamma_0}
    Y_{u,\sigma(i)}Y_{v,\sigma(i)}$,
which is the sum of terms
$\{Y_{u,\sigma(i)}Y_{v,\sigma(i)}\}_{i=1}^{\Gamma_0}$ that are each bounded
within $[-1,1]$ and drawn without replacement from a population of all
possible items. Hoeffding's inequality (which also applies to the current scenario of sampling without replacement \cite{hoeffding1963}) yields
\begin{align}
\P\big(
    \langle \widetilde{Y}_u, \widetilde{Y}_v \rangle
    \ge \theta\Gamma_0
    \mid
    p_u \ne p_v
  \big)
\le
  \exp\left(
        -\frac{\big(
                 \theta\Gamma_0
                 -
                 \E[\langle
                      \widetilde{Y}_u, \widetilde{Y}_v
                    \rangle
                    \mid p_u \ne p_v]
               \big)^2}{2\Gamma_0}
      \right).
\label{eq:diff-source-inner-prod-hoeffding}
\end{align}
By $\gamma$-incoherence and our choice of
$\theta$,
\begin{equation}
\theta\Gamma_0
-\E\big[
     \langle \widetilde{Y}_u, \widetilde{Y}_v \rangle 
     \mid
     p_u \ne p_v
   \big] 
\ge
  \theta\Gamma_0 - 4\gamma\Delta^2\Gamma_0
=
   (\theta-4\gamma\Delta^2)\Gamma_0
\ge
   0.
\label{eq:diff-source-hoeffding-var}
\end{equation}
Above, we used the fact that $\Gamma_0$ randomly explored items are a 
random subset of $m$ items, and hence 
\begin{align*}
\E\big[\tfrac1\Gamma_0\langle  \widetilde{Y}_u, \widetilde{Y}_v\rangle\big] & =  \E\big[\tfrac1m \langle  {Y}_u, {Y}_v\rangle\big],
\end{align*}
with $Y_u, Y_v$ representing the entire (random) vector of preferences of $u$ and $v$ respectively.

Combining inequalities \eqref{eq:diff-source-inner-prod-hoeffding} and
\eqref{eq:diff-source-hoeffding-var}
yields
\[
\P\big(
    \langle \widetilde{Y}_u, \widetilde{Y}_v \rangle
    \ge \theta\Gamma_0
    \mid
    p_u \ne p_v
  \big)
\le
  \exp\Big( -\frac{(\theta-4\gamma\Delta^2)^2}{2}\Gamma_0 \Big).
\]
A similar argument as the ending of Lemma
\ref{lem:same-source-pairwise-error}'s proof establishes that the bound
holds even if users $u$
and $v$ have jointly explored more than $\Gamma_0$ items.
\end{proof}


We now prove Lemma~\ref{lem:num-bad-neighbors-bound}. 

Let $\beta = \exp(-(\theta-4\gamma\De^2)^2 \Gamma_0/2)$ be the probability
bound from Lemma~\ref{lem:diff-source-pairwise-error}, where by our choice of
$\theta = 2 \Delta^2 (1 + \gamma)$ and with $\Gamma_0=t^{1-\alpha}/2$, we have
$\beta = \exp(-\Delta^4(1-\gamma)^2 t^{1-\alpha})$.

By Lemma~\ref{lem:diff-source-pairwise-error}, for a pair of users $u$ and $v$
with at least $t^{1-\alpha}/2$ items jointly explored, the probability that
they are erroneously declared neighbors is upper-bounded by $\beta$.

Denote the set of users of type different from $u$ by $\BB$, and write
$$
n_\text{bad}
= \sum_{v\in \BB}
    \ind\{ u \text{ and } v \text{ are declared to be neighbors}\},
$$
whence $\E[n_\text{bad}] \le n \beta$. Markov's inequality gives
\[
\P( n_{\text{bad}} \ge n \sqrt{\beta} )
\le \frac{\E[n_{\text{bad}}]}
         {n \sqrt{\beta}}
\le \frac{n \beta}
         {n \sqrt{\beta}}
= \sqrt{\beta}\,,
\]
proving the lemma.

\subsection{Proof of Lemma~\ref{lem:final}}\label{appendix:lemmafinal}

We reproduce Lemma~\ref{lem:final} below.

\newtheorem*{lem-final}{Lemma \ref{lem:final}}
\begin{lem-final}
For user $u$ at time $t$, if the good neighborhood event $\En(u,t)$ holds
and $t \le \mu m$, then
\begin{align*}
\P(X_{ut}=1) 
&\ge
   1
   - 2m\exp\Big(-\frac{\Delta^2 t n^{1-\alpha}}{40km}\Big)
   - \frac{1}{t^\alpha} - \frac{1}{n^\alpha}\,.
\end{align*}
\end{lem-final}

We begin by checking that when the good neighborhood event $\En(u,t)$
holds for user $u$, the items have been rated
enough times by the good neighbors.

\begin{lemma}
\label{lem:item-exploration}
For user $u$ at time $t$, suppose that the good neighborhood event
$\En(u,t)$ holds. Then for a given item $i$,
\[
\P\Big( \text{item }i\text{ has}\le\frac{t n^{1-\alpha}}{10km}
        \text{ ratings from good neighbors of }u\Big)
\le \exp\Big( -\frac{t n^{1-\alpha}}{40km} \Big).
\]
\end{lemma}
\begin{proof}
The number of user $u$'s good neighbors who have rated item $i$ stochastically
dominates a $\text{Bin}( \frac{n}{5k}, \frac{\varepsilon_R(n)t}{m} )$ random
variable, where $\frac{\varepsilon_R(n)t}{m}=\frac{t}{m n^\alpha}$
(here, we have critically used the lower bound on the number of good neighbors
user $u$ has when the good neighborhood event $\En(u,t)$ holds). By a Chernoff
bound,
\[
\P\bigg(
    \text{Bin}\Big(
                \frac{n}{5k},
                \frac{t}{m n^\alpha}
              \Big)
    \le \frac{tn^{1-\alpha}}{10km}
  \bigg)
\le
  \exp\bigg(
        -\frac{1}{2}
        \frac{(\frac{t n^{1-\alpha}}{5km}-\frac{t n^{1-\alpha}}{10km})^2}
              {\frac{t n^{1-\alpha}}{5km}}
      \bigg)
\le
   \exp\Big( -\frac{t n^{1-\alpha}}{40km} \Big). \qedhere
\]
\end{proof}

Next, we show a sufficient condition for which the algorithm correctly
classifies every item as likable or unlikable for user $u$.

\begin{lemma}
\label{lem:item-classification}
Suppose that there are no $\Delta$-ambiguous items. For user $u$ at time
$t$, suppose that the good neighborhood event $\En(u,t)$ holds. Provided that
every item $i\in[m]$ has more than $\frac{t n^{1-\alpha}}{10km}$ ratings from
good neighbors of user $u$, then with probability at least
$1-m\exp(-\frac{\Delta^2 t n^{1-\alpha}}{20km})$, we have that for every item
$i\in[m]$,
\begin{align*}
\widetilde{p}_{ui}
&> \frac{1}{2}\quad\text{if item }i\text{ is likable by user }u, \\
\widetilde{p}_{ui}
&< \frac{1}{2}\quad\text{if item }i\text{ is unlikable by user }u.
\end{align*}
\end{lemma}
\begin{proof}
Let $A$ be the number of ratings that good neighbors of user $u$ have
provided. 
Suppose item $i$ is likable by user
$u$. Then when we condition on
$A = a_0 \triangleq \lceil \frac{t n^{1-\alpha}}{10km} \rceil$,
$\widetilde{p}_{ui}$ stochastically dominates
\[
q_{ui}
\triangleq
  \frac{\text{Bin}(a_0, p_{ui})}{a_0 + \Delta a_0}
= \frac{\text{Bin}(a_0, p_{ui})}{(1 + \Delta) a_0},
\]
which is the worst-case variant of $\widetilde{p}_{ui}$ that insists that all
$\Delta a_0$ bad neighbors provided rating ``$-1$'' for likable item $i$
(here, we have critically used the upper bound on the number of bad neighbors
user $u$ has when the good neighborhood event $\En(u,t)$ holds).
Then
\begin{align*}
\P( q_{ui} \le \frac{1}{2} \mid A=a_0 )
&=\P\bigg(
      \text{Bin}(a_0, p_{ui})
      \le
        \frac{(1+\Delta)a_0}{2}
      \;\bigg|\;
      A=a_0
    \bigg) \\
&=\P\bigg(
      a_0 p_{ui} - \text{Bin}(a_0, p_{ui})
      \ge
        a_0 \Big(
              p_{ui} - \frac{1}{2} - \frac{\Delta}{2}
            \Big)
      \;\bigg|\;
      A=a_0
    \bigg) \\
&\overset{(i)}{\le}
   \exp\Big(
         -2a_0 \Big(p_{ui} - \frac{1}{2} - \frac{\Delta}{2}\Big)^2
       \Big) \\
&\overset{(ii)}{\le}
   \exp\Big( -\frac{1}{2} a_0 \Delta^2 \Big) \\
&\overset{(iii)}{\le}
   \exp\Big( -\frac{\Delta^2 t n^{1-\alpha}}{20km} \Big),
\end{align*}
where step $(i)$ is Hoeffding's inequality, step $(ii)$ follows from
item $i$ being likable by user $u$ (i.e., $p_{ui} \ge \frac{1}{2}+\Delta$),
and step $(iii)$ is by our choice of $a_0$. Conclude then that
\[
\P(\widetilde{p}_{ui} \le \frac{1}{2} \mid A=a_0)
\le
  \exp\Big( -\frac{\Delta^2 t n^{1-\alpha}}{20km} \Big).
\]
Finally,
\begin{align*}
\P\Big(
    \widetilde{p}_{ui} \le \frac{1}{2}
    \;\Big|\;
    A \ge \frac{tn^{1-\alpha}}{10km}
  \Big)
&=\frac{\sum_{a=a_0}^{\infty}
          \P(A=a)\P(\widetilde{p}_{ui} \le \frac{1}{2} \mid A=a)}
       {\P( A\ge\frac{t n^{1-\alpha}}{10km} )} \\
&\le
  \frac{\sum_{a=a_0}^{\infty}
          \P(A=a)\exp( -\frac{\Delta^2 t n^{1-\alpha}}{20km})}
       {\P(A \ge \frac{t n^{1-\alpha}}{10km})} \\
&=\exp\Big( -\frac{\Delta^2 t n^{1-\alpha}}{20km} \Big).
\end{align*}
A similar argument holds for when item $i$ is unlikable. Union-bounding
over all $m$ items yields the claim.
\end{proof}

We now prove Lemma \ref{lem:final}. First off, provided that $t \le \mu m$, we
know that there must still exist an item likable by user $u$ that user $u$
has yet to consume. For user $u$ at time $t$, supposing that event $\En(u,t)$
holds, then every item has been rated more than $\frac{t n^{1-\alpha}}{10km}$
times by the good neighbors of user $u$ with probability at least
$1 - m\exp(-\frac{t n^{1-\alpha}}{40km})$. This follows from
union-bounding over the $m$ items with Lemma~\ref{lem:item-exploration}.
Applying Lemma~\ref{lem:item-classification}, and noting that we only exploit
with probability
$1-\varepsilon_J(t)-\varepsilon_R(n)=1-1/t^\alpha-1/n^\alpha$,
we finish the proof:
\begin{align*}
\P(X_{ut}=1)
&\ge
   1 - m \exp\Big(-\frac{t n^{1-\alpha}}{40km}\Big)
   - m \exp\Big(-\frac{\Delta^2 t n^{1-\alpha}}{20km} \Big)
   - \frac{1}{t^\alpha} - \frac{1}{n^\alpha} \\
&\ge
   1
   - 2m \exp\Big(-\frac{\Delta^2 t n^{1-\alpha}}{40km} \Big)
   - \frac{1}{t^\alpha} - \frac{1}{n^\alpha}.
\end{align*}

\subsection{Experimental Results}
\label{sec:appendix-experimental-results}

We demonstrate our algorithm
\textsc{Collaborative-Greedy} on
two datasets, showing that they have comparable performance and that they
both outperform two existing recommendation algorithms Popularity Amongst
Friends (PAF) \cite{paf} and Deshpande and Montanari's method (DM)
\cite{deshpande_montanari}. At each time step, PAF finds nearest neighbors
(``friends'') for every user and recommends to a user the ``most popular''
item, i.e., the one with the most number of $+1$ ratings, among the user's
friends. DM doesn't do any collaboration beyond a preprocessing step that
computes item feature vectors via matrix completion. Then during online
recommendation, DM learns user feature vectors over time with the help of
item feature vectors and recommends an item to each user based on whether it
aligns well with the user's feature vector.

We simulate an online recommendation system based on movie ratings from the
Movielens10m and Netflix datasets, each of which provides a sparsely filled
user-by-movie rating matrix with ratings out of 5 stars. Unfortunately,
existing collaborative filtering datasets such as the two we consider don't
offer the interactivity of a real online recommendation system, nor do they
allow us to reveal the rating for an item that a user didn't actually rate.
For simulating an online system, the former issue can be dealt with by
simply revealing entries in the user-by-item rating matrix over time. We
address the latter issue by only considering a dense ``top users vs.~top
items'' subset of each dataset. In particular, we consider only the ``top''
users who have rated the most number of items, and the ``top'' items that have
received the most number of ratings. While this dense part of the dataset is
unrepresentative of the rest of the dataset, it does allow us to use actual
ratings provided by users without synthesizing any ratings.

An initial question to ask is whether the dense movie ratings matrices we
consider could be reasonably explained by our latent source model. We
automatically learn the structure of these matrices using the method by
Grosse et al.~\cite{rgrosse_best_paper_2012} and find Bayesian clustered
tensor factorization (BCTF) to accurately model the data. This finding isn't
surprising as BCTF has previously been used to model movie ratings
data~\cite{bctf_2009}. BCTF effectively clusters both users and movies so
that we get structure such as that shown in Figure
\ref{fig:experimental-results}\subref{fig:bctf} for the
Movielens10m ``top users vs.~top items'' matrix. Our latent source model
could reasonably model movie ratings data as it only assumes clustering of
users.


Following the experimental setup of \cite{paf},
we quantize a rating of 4 stars or more to be $+1$ (likeable), and a rating
of 3 stars or less to be $-1$ (unlikeable). While we look at a dense subset
of each dataset, there are still missing entries. If a user $u$ hasn't rated
item $j$ in the dataset, then we set the corresponding true rating to 0,
meaning that in our simulation, upon recommending item $j$ to user $u$, we
receive 0 reward, but we still mark that user $u$ has consumed item $j$; thus,
item $j$ can no longer be recommended to user $u$. For both Movielens10m and
Netflix datasets, we consider the top $n=200$ users and the top $m=500$
movies. For Movielens10m, the resulting user-by-rating matrix has 80.7\%
nonzero entries. For Netflix, the resulting matrix has 86.0\% nonzero entries.
For an algorithm that recommends item $\pi_{ut}$ to user $u$ at time $t$, we
measure the algorithm's average cumulative reward up to time $T$ as
\begin{equation*}
\frac{1}{n} \sum_{t=1}^T\sum_{u=1}^n Y_{u \pi_{ut}}^{(T)},
\end{equation*}
where we average over users.

For all methods, we recommend items until we reach time $T=500$, 
i.e., we make movie recommendations until each user
has seen all $m=500$ movies. 
We disallow the matrix completion step for DM to see the users that we
actually test on, but we allow it to see the the same items as what is in the
simulated online recommendation system in order to compute these items'
feature vectors (using the rest of the users in the dataset). Furthermore,
when a rating is revealed, we provide DM both the thresholded rating and the
non-thresholded rating, the latter of which DM uses to estimate user feature
vectors over time.

Parameters $\theta$ and $\alpha$ for 
and \textsc{Collaborative-Greedy} are chosen using training data: We sweep over
the two parameters on training data consisting of 200 users that are the
``next top'' 200 users, i.e., ranked 201 to 400 in number movie ratings they
provided. For simplicity, we discretize our search space to
$\theta\in\{0.0, 0.1, \dots, 1.0\}$ and
$\alpha\in\{0.1, 0.2, 0.3, 0.4, 0.5\}$. We choose the parameter setting
achieving the highest area under the cumulative reward curve. For both
Movielens10m and Netflix datasets, this corresponded to setting
$\theta=0.0$ and $\alpha=0.5$ for \textsc{Collaborative-Greedy}. In contrast,
the parameters for PAF and DM are chosen to be the best parameters for the
test data among a wide range of parameters. The results are shown in Figure
\ref{fig:supp-results-top}. We find that our algorithm
\textsc{Collaborative-Greedy} 
outperforms PAF and DM. We remark that the curves
are roughly concave, which is expected since once we've finished recommending
likeable items (roughly around time step 300), we end up recommending mostly
unlikeable items until we've exhausted all the items.

\begin{figure}
\captionsetup[subfloat]{farskip=0em,captionskip=-.4em,nearskip=-5em}
\noindent
\centering
\subfloat[]{
\includegraphics[width=7.5cm]{figs/nips_resubmit_movielens10m_top.pdf}
} \\
\subfloat[]{
\includegraphics[width=7.5cm]{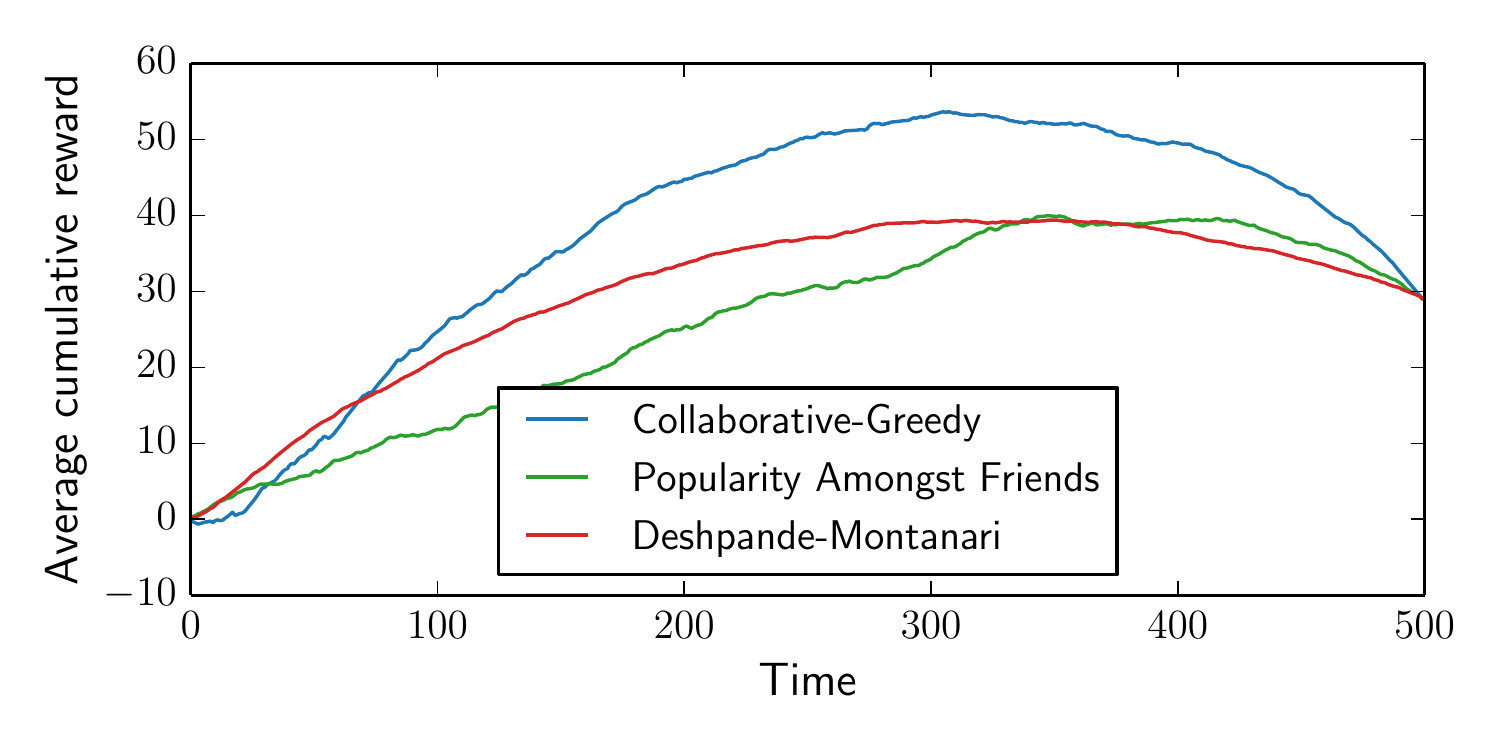}
}
\caption{Average cumulative rewards over time: 
(a) Movielens10m, (b) Netflix.}
\label{fig:supp-results-top}
\end{figure}

\end{document}